\documentclass{article}

\usepackage[preprint]{neurips_2026}
\usepackage{multirow}
\usepackage{makecell}
\usepackage{tcolorbox}
\usepackage[utf8]{inputenc}
\usepackage[T1]{fontenc}
\usepackage{hyperref}
\usepackage{url}
\usepackage{booktabs}
\usepackage{amsfonts}
\usepackage{amsmath}
\usepackage{amssymb}
\usepackage{microtype}
\usepackage{xcolor}
\usepackage{wrapfig}
\usepackage{graphicx}
\usepackage{multirow}
\usepackage{colortbl}
\usepackage{xcolor}
\definecolor{lightblue}{RGB}{220,230,250}
\usepackage{amsthm}
\renewcommand{\citet}{\cite}
\newtheorem{theorem}{Theorem}[section]   
\newtheorem{proposition}[theorem]{Proposition}

\theoremstyle{definition}

\usepackage{algorithm}
\usepackage{algpseudocode}

\title{Continual Fine-Tuning of Large Language Models via Program Memory}

\author{%
  Hung Le \\
  Deakin Applied AI Initiative, Deakin University, Australia \\
  \texttt{thai.le@deakin.edu.au} \\
  \And
  Svetha Venkatesh\\
  Deakin Applied AI Initiative, Deakin University, Australia \\
  \texttt{svetha.venkatesh@deakin.edu.au} \\
}

\begin{document}

\maketitle

\begin{abstract}
Parameter-Efficient Fine-Tuning (PEFT), particularly Low-Rank Adaptation (LoRA), has become a standard approach for adapting Large Language Models (LLMs) under limited compute. However, in continual settings where models are updated sequentially with small datasets, conventional LoRA updates struggle to balance rapid adaptation and knowledge retention. Existing methods typically treat the low-rank space as a homogeneous update region, lacking mechanisms to regulate how short-term updates are consolidated over time. We propose a continual LoRA framework with \textbf{Pro}gram memory, inspired by \textbf{C}omplementary \textbf{L}earning Systems in neuroscience. Our approach, dubbed \textbf{ProCL}, organizes LoRA adapters into structured program memory slots that are dynamically retrieved through input-conditioned attention. This enables rapid and localized adaptation, encouraging similar inputs to reuse shared adapter regions while reserving unused capacity for future data. The slots are then combined with the underlying adapter, which maintains a distributed representation that gradually accumulates knowledge across tasks to balance plasticity and stability. Our method operates entirely within the LoRA parameterization and incurs no additional inference cost. Experiments on diverse benchmarks demonstrate improved retention and reduced catastrophic forgetting over other continual LoRA strategies.
\end{abstract}

\section{Introduction}

Large language models (LLMs) have achieved strong performance across a wide range of tasks \citep{brown2020language, achiam2023gpt, le2025multi, do2025automatic, yang2025superrag}. However, adapting them to evolving real-world scenarios remains challenging, as full fine-tuning is often computationally prohibitive. Parameter-Efficient Fine-Tuning (PEFT), particularly Low-Rank Adaptation (LoRA) \citep{hulora}, addresses this by restricting updates to low-rank subspaces while freezing the backbone, enabling efficient adaptation with minimal additional parameters.

\begin{figure}[t]
    \centering
    \includegraphics[width=1\linewidth]{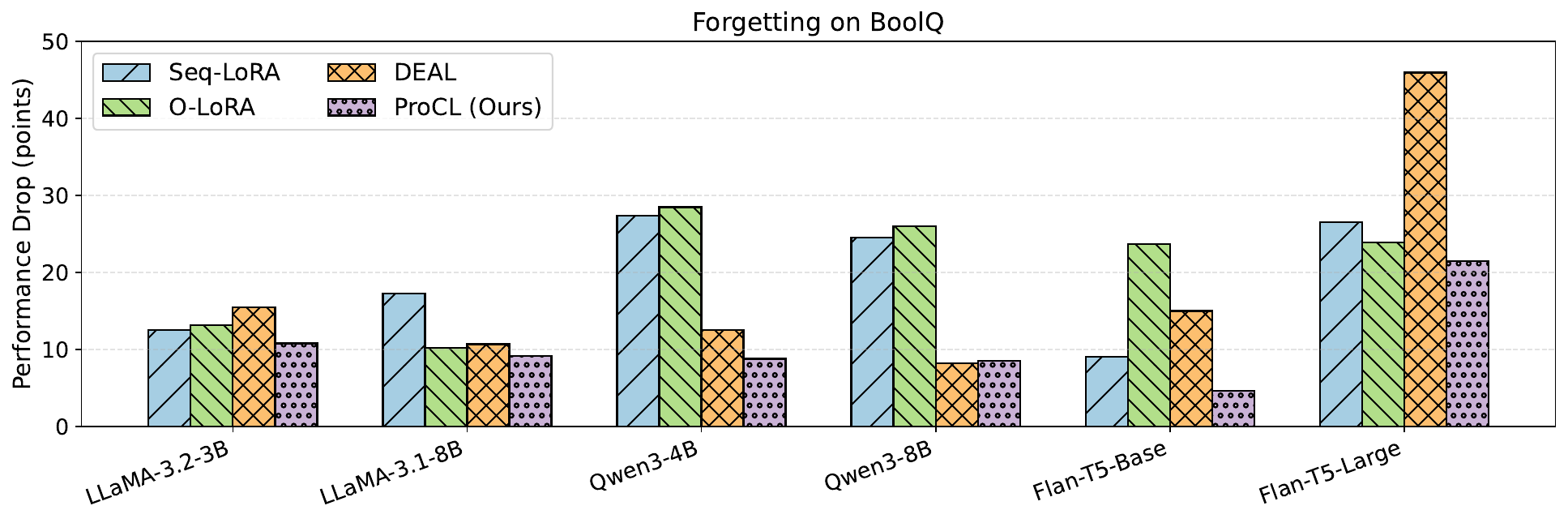}
    \caption{LLMs' catastrophic forgetting in QA tasks (lower is better). BoolQ's accuracy drops after fine-tuning on SQuAD and AdversarialQA. While all methods suffer from forgetting across different LLM backbones, our approach dominantly exhibits the smallest performance degradation.}
    \label{fig:forgetting-bool}
\end{figure}

In real-world deployment, model adaptation is rarely a one-time process; LLMs must continually incorporate new knowledge from sequential, often small-scale datasets. For instance, a clinical model fine-tuned for diagnosis may later be updated for treatment recommendation, improving new capabilities while degrading earlier ones. This highlights the stability–plasticity dilemma, where models must learn new tasks without forgetting prior knowledge \citep{mccloskey1989catastrophic}. Although LoRA constrains updates to low-rank subspaces, sequential fine-tuning can still introduce interference, leading to forgetting and reduced retention \citep{ding2024boosting, han2025data}.

Existing approaches for continual adaptation of LLMs often add adapters or task-specific modules to isolate task-relevant knowledge \citep{ren2024melora, tian2025adapters}. While effective at reducing interference, these methods increase architectural complexity, memory usage, and inference latency, making them cumbersome for long task sequences. Alternatively, regularization-based techniques constrain updates to orthogonal subspaces to reduce interference across tasks \citep{wang2023orthogonal, yang2025parameter}. While effective at mitigating forgetting, the model's effective plasticity is reduced, limiting its ability to learn new information on later tasks. Recent strategies improve retention by refining low-rank factors with wavelet kernels \citep{han2025data}. However, they still modify the shared low-rank subspace globally, without explicitly controlling how new knowledge is written, partitioned, or consolidated. As a result, LoRA updates remain vulnerable to interference, leading to degraded performance on previously learned tasks. Empirically, prior work mainly examines simple text classification tasks, which neither reflect the generative capabilities of LLMs nor fully capture catastrophic forgetting in more complex settings. For example, in question answering, we consistently observe substantial performance degradation of prior methods after only a few tasks (see Fig. \ref{fig:forgetting-bool}). These limitations raise a key question: \textit{can LoRA be equipped with a principled mechanism for efficient and robust continual learning?}

To address this challenge, we propose a memory-structured adaptation mechanism that organizes how knowledge is written into LoRA adapters during continual learning. Drawing inspiration from Complementary Learning Systems (CLS) in neuroscience \citep{mcclelland1995there, norman2003modeling}, we model continual adaptation as the interaction between rapid episodic indexing and slower distributed representations. In our framework, dubbed \textbf{ProCL}, low-rank adapters are treated as program memory \citep{leneural, le2022neurocoder}, with the adapter weights partitioned into multiple slots or programs. An input-conditioned attention mechanism dynamically retrieves and composes these programs, allowing similar inputs to activate and update shared adapter regions while routing unrelated inputs to different parts of the parameter space. This attention-based routing localizes parameter updates, reducing interference between different data regions.

At the same time, the original LoRA weight provides a stable distributed representation that reflects previously learned knowledge. During training, this weight is combined with the program-composed weights to produce the executing adapter, serving as a shared baseline that stabilizes optimization while the programs introduce input-dependent adaptations. These adaptive compositions are periodically integrated into a persistent adapter through a consolidation step, which accumulates routing-based updates over time. During inference, this learned adapter is used directly without program routing. By structuring adaptation in this way, ProCL enables selective plasticity during training while producing a compact, static model for deployment, mitigating catastrophic forgetting without introducing additional inference overhead. Experiments on text classification and question answering continual learning benchmarks demonstrate that this memory-structured adaptation mechanism significantly improves knowledge retention, cross-task generalization, and  efficiency compared to existing continual LoRA strategies.

\section{Related Work}

Continual learning (CL) aims to adapt models to new tasks sequentially without catastrophic forgetting \citep{mccloskey1989catastrophic}. Applying CL to LLMs is challenging due to their massive parameter spaces, high computational costs, and sensitivity to interference across tasks. Existing approaches for continuous LLM fine-tuning can be grouped into two main directions.  

\textbf{Adapter‑Modular and Progressive Architectures.}  
Adapter‑based and parameter‑efficient tuning approaches \citep{houlsby2019parameter, hulora} freeze the backbone and introduce task modules. Progressive or modular extensions, including LoraHub \citep{huanglorahub} and MOLE \citep{wumixture}, dynamically compose or gate multiple LoRA components for cross‑task reuse. Although these designs reduce catastrophic forgetting, they increase memory footprint and complexity as tasks accumulate, assuming knowing the tasks in inference.

\textbf{Regularization-based Methods.}  
Regularization and subspace-based approaches aim to mitigate interference across tasks. Classical methods such as EWC \citep{kirkpatrick2017overcoming} constrain updates based on parameter importance, while more recent techniques focus more on LoRA-specific constraints. For example, O-LoRA \citep{wang2023orthogonal} enforces orthogonality between low-rank updates, N-LoRA \citep{yang2025parameter} prevents parameter collisions via disjoint subspaces, OLieRA \citep{cao2025orthogonal} preserves model geometry through Lie group constraints, and DEAL \citep{han2025data} uses wavelet kernels to retain historical knowledge.  
Despite these advances, these methods lack explicit mechanisms to isolate task-specific updates, and accumulated constraints gradually exhaust the effective parameter space, limiting plasticity over long task sequences. In contrast, inspired by CLS \citep{mcclelland1995there, norman2003modeling}, our approach introduces program-memory slots to structure how task knowledge is written and consolidated, enabling selective retention and adaptability.

\textbf{Memory in CL.}  
Memory-based methods mitigate forgetting by storing past information, e.g., via experience replay \citep{rolnick2019experience} and gradient episodic memory \citep{lopez2017gradient, chaudhry2019tiny, hoanguniversal}. Another line explores program memory, where reusable weight components are stored and composed across tasks \citep{leneural, lestable, le2022neurocoder}. Related to this, mixture-of-experts (MoE, \citep{shazeer2017outrageously}) models route inputs to specialized experts, but combine \emph{outputs}, whereas program memory methods compose \emph{parameters} to form task-adaptive models. Prior approaches rely on external buffers or modules and are not tailored for parameter-efficient LLM adaptation. In contrast, we integrate program memory directly into LoRA, enabling structured storage, routing, and consolidation within the adapter without external data buffers.

\section{Background}

\textbf{LoRA-Based Fine-tuning.} LoRA (Low-Rank Adaptation) enhances LLMs by introducing low-rank matrices to parameter updates during fine-tuning. 
Given a pre-trained weight matrix $W_0 \in \mathbb{R}^{D_1 \times D_2}$, LoRA decomposes the weight update $\Delta W$ into two smaller matrices $A \in \mathbb{R}^{D_1 \times R}$ and $B \in \mathbb{R}^{D_2 \times R}$, where $R \ll \min\{D_1,D_2\}$. The update is expressed as:
\begin{equation}
\Delta W = A B^\top
\end{equation}
This reduces the number of trainable parameters from $D_1 \cdot D_2$ to $D_1 \cdot R + D_2 \cdot R$, enabling efficient adaptation to new tasks with significantly lower computational cost than full fine-tuning. In practice, the adapted linear transformation is applied as $y = x(W_0 + \Delta W)^\top$. Note that the factorization $\Delta W = AB^\top$ is equivalent to the commonly used form $BA$ up to transposition conventions, and does not affect the forward computation. In our method, we operate on a generic low-rank factor (e.g., $A$ or $B^\top$), which we denote as $W$ for simplicity, without loss of generality.

\textbf{Continual Learning with LoRA.} Formally, in the general continual learning setting, consider a sequence of tasks $\mathcal{T}_1, \dots, \mathcal{T}_t$ with corresponding datasets $\mathcal{D}_1, \dots, \mathcal{D}_t$. The model is presented with the task datasets sequentially, one at a time.
The standard objective is to learn a model that performs well across all seen tasks. At task $t$, the objective can be expressed as:

\begin{equation}
\min_{\theta_1, \dots, \theta_t} \frac{1}{t} \sum_{i=1}^t \mathcal{L}(\theta_i, \mathcal{D}_i),
\end{equation}

where $\theta_i$ denotes the model parameters after learning task $\mathcal{T}_i$, and $\mathcal{L}$ is the task-specific loss. 

In the context of LoRA-based fine-tuning, the model parameters $\theta_i$ at task $i$  are decomposed into a fixed pre-trained weight $W_0$ and low-rank adaptation matrices $A_i, B_i$ for each task. The continual learning objective then becomes:

\begin{equation}
\min_{A_1, B_1, \dots, A_t, B_t} \frac{1}{t} \sum_{i=1}^t \mathcal{L}\big(W_0 + A_i B_i^\top, \mathcal{D}_i\big),
\end{equation}

Here, the goal is to find low-rank matrices $A_i$ and $B_i$ that allow the model to efficiently adapt to each new task while preserving performance on all previously learned tasks, naturally reflecting the stability–plasticity trade-off central to continual learning. For example, a simple method (Sequential LoRA) shares the same adapters across tasks, i.e., they are initialized from the previously learned matrices: $A_i \gets A_{i-1}$ and $B_i \gets B_{i-1}$, to better accumulate knowledge from past tasks.

\section{Method}

We propose \textbf{ProCL}, a \textbf{Pro}gram Memory framework for LoRA inspired by \textbf{C}omplementary \textbf{L}earning systems. ProCL addresses catastrophic forgetting in continual fine-tuning by replacing a single monolithic LoRA adapter with a library of small specialist weight matrices, termed programs. The input is compressed into a summary vector and passed through a lightweight encoder to produce a data fingerprint, which is compared against learnable program keys via multi-head attention to produce soft routing weights. These weights determine how programs are combined to form the final adapter. As training progresses, different tasks activate different programs, reducing gradient interference; we show formally that such interference is attenuated and approaches zero as routing specialisation sharpens. The final adapter combines a frozen copy of the original adapter with the dynamically routed program blend, preventing overwriting of prior knowledge. Analogous to biological memory consolidation, a moving average step folds the executed adapter back into the persistent weight, gradually integrating new knowledge into a stable representation. The process is depicted in Fig.~\ref{fig:procl}. 

\begin{figure}[t]
\centering
\includegraphics[width=\linewidth]{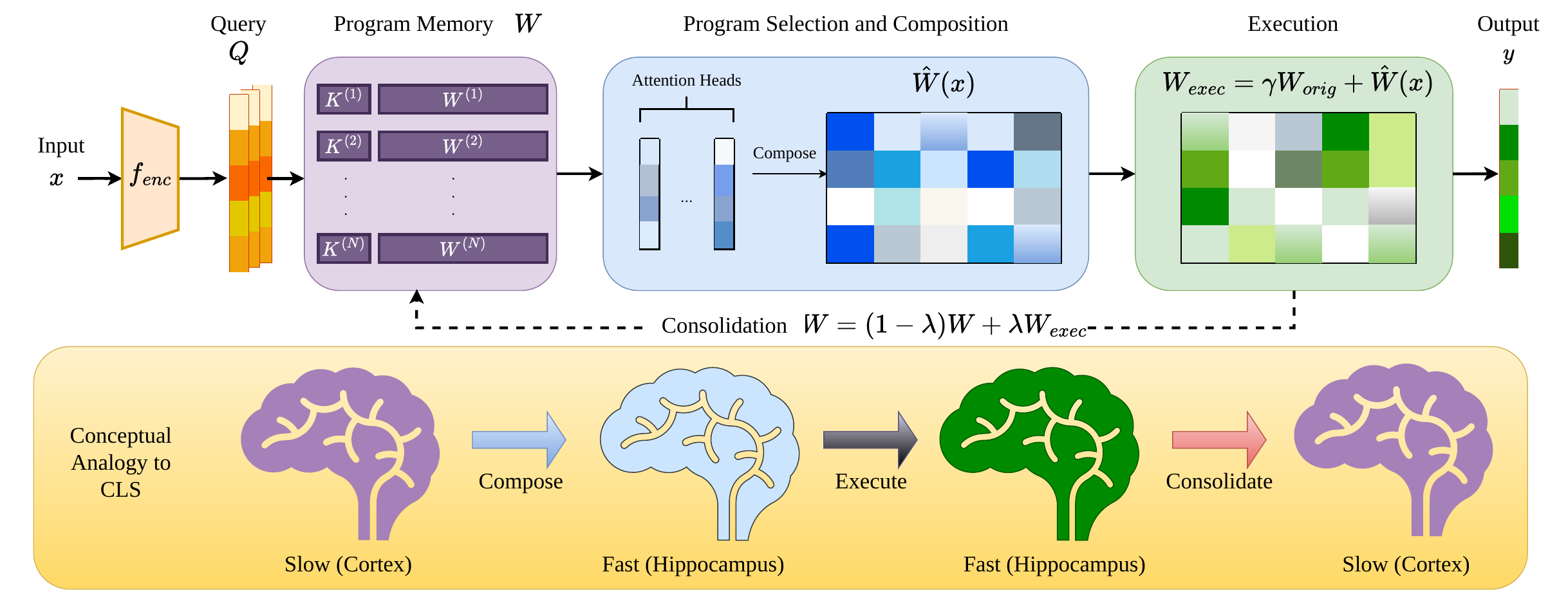}
\caption{Overview of ProCL. Given an input $x$, the model queries relevant parameter slots in the Program Memory $W$. Through attention, selected programs are composed into an input-conditioned weight $\hat{W}(x)$, which is combined with the base weight for execution to produce the output $y$. The resulting execution weight is then distilled back into the memory via consolidation. This process is conceptually inspired by Complementary Learning Systems (CLS).}
\label{fig:procl}
\end{figure}

\subsection{Program Memory as Structured Low-Rank Adapters}

Consider a low-rank adapter weight:
\begin{equation}
W \in \mathbb{R}^{R \times D},
\end{equation}
where $R$ is the LoRA rank and $D$ the other dimension. Instead of treating $W$ as a monolithic matrix, we partition it into $N$ programs:
\begin{equation}
W = \left[W^{(1)}, \dots, W^{(N)}\right],
\quad
W^{(n)} \in \mathbb{R}^{r \times D},
\quad
r = R / N.\label{eq:pm}
\end{equation}
Here, $n=1,\dots,N$ with $N$ as a hyperparameter dividable by $R$. Each program $W^{(n)}$ serves as a memory slot that stores data-relevant adaptations. We use the terms \emph{program} and \emph{memory slot} interchangeably throughout the paper. Importantly, the adapter weight $W$ remains the only persistent parameter saved by LoRA frameworks. Program memory is not a separate parameter set; it is an execution-time abstraction over the same adapter tensor $W$, ensuring full compatibility with existing LoRA implementations.

\subsection{Input-Conditioned Program Selection}

During training, given a training sample: $x \in \mathbb{R}^{T \times D}$, where $T$ and $D$ are the number of tokens and embedding size, respectively, we compute an input representation via mean pooling:
\begin{equation}
z = \frac{1}{T} \sum_{t=1}^{T} x_t.
\end{equation}

A lightweight task encoder produces task keys:
\begin{equation}
Q = f_{\text{enc}}(z) \in \mathbb{R}^{N \times d_k}.
\end{equation}

Inspired by \citet{leneural}, we use the task queries to perform attention over the programs. Attending to slot-based memory resembles the episodic retrieval mechanism in CLS \citep{le2021model, whittington2025tale}. Since each attention operation selects a program $W^{(n)} \in \mathbb{R}^{r \times D}$, we employ $R / r=N$ parallel attention heads to retrieve $N$ programs, which are then concatenated to form the executed adapter. Each head $h\in[1,N]$ maintains a set of learnable key vectors $K_h=\{K_h^{1},\dots,K_h^{N}\}\in \mathbb{R}^{N \times d_k}$ associated with the set of programs in Eq. \ref{eq:pm}. The attention weights for head $h$ over memory slots are computed as:
\begin{equation}
\alpha_h = \text{Softmax}\left(Q[h] K_h^\top\right),
\end{equation}
where $\alpha_h \in \mathbb{R}^{N}$. To improve stability during training, we aggregate attention across the batch:
\begin{equation}
\bar{\alpha}_h = \frac{1}{B}\sum_{b=1}^{B} \alpha_h^{(b)}.
\end{equation}

This yields a batch-level routing distribution that remains input-adaptive across batches while reducing variance from per-sample fluctuations. Since batches in continual fine-tuning predominantly contain samples from the same task, this aggregation preserves task-conditioned routing behavior where different tasks still activate different programs. The effective program for head $h$ is obtained via a soft composition over memory slots:
\begin{equation}
\hat{W}_h =
\sum_{n=1}^{N}
\bar{\alpha}_{h,n}
\cdot
\sigma(s_n)
\cdot
W^{(n)},
\end{equation}
where $s_n$ is a learnable scaling parameter and $\sigma(\cdot)$ ensures stable gating.
We emphasize that this operation performs a \emph{soft selection} over programs, allowing multiple memory slots to contribute to the composed adapter. Then, all head-specific programs are concatenated:
\begin{equation}
\hat{W} = \left[\hat{W}_1, \dots, \hat{W}_N\right].\label{eq:What}
\end{equation}

Theoretically, our proposed mechanism reduces cross-task gradient interference by routing gradient updates through input-conditioned programs. The theoretical result is summarized below:

\begin{tcolorbox}[colback=gray!5,colframe=black]
\textbf{Remark 1: Program Routing Bounds Cross-Task Gradient Interaction.}
Let $J(t, t') := |\langle G^t, G^{t'} \rangle_F|$ denote the magnitude
of cross-task gradient interaction. ProCL satisfies
$J_{\mathrm{ProCL}}(t, t') \le (\max_h \beta_h)\, S(t, t')$, where
$\beta_h \in [0,1]$ is the routing-induced gating coefficient and
$S(t, t') := \sum_h |\langle \tilde{g}_{h,t}, \tilde{g}_{h,t'}
\rangle_F|$. Under routing specialization (A1), $\max_h \beta_h \to 0$
as task-specific routing distributions become disjoint, driving
$J_{\mathrm{ProCL}}(t, t') \to 0$.
\end{tcolorbox}

\begin{proof}[Proof sketch]
Each program receives a gradient scaled by
$\beta_h = \sum_n \sigma(s_{n,t}) \sigma(s_{n,t'})\, \bar{\alpha}_{h,n,t}\,
\bar{\alpha}_{h,n,t'}$, with $0 \le \beta_h \le 1$ since routing
weights are normalised and $\sigma(\cdot) < 1$. The triangle inequality
then bounds the cross-task interaction magnitude by $(\max_h \beta_h)\,
S(t, t')$. Under (A1), routing distributions concentrate on disjoint
subsets, forcing $\beta_h \to 0$. See Appendix~\ref{app:gradient_interference} for full theory and proof.
\end{proof}

\subsection{Stable–Adaptive Weight Composition and Consolidation}\label{subsec:Combining}

The composed programs are aggregated with the original adapter to form the executing adapter:
\begin{equation}
W_{\text{exec}} = \gamma W_{\text{orig}} + \hat{W}(x),
\end{equation}
where \(W_{\text{orig}}\) denotes the initial LoRA adapter weight before task-specific training, \(\hat{W}(x)\) is the program-composed weight produced by the input-conditioned routing mechanism (Eq. \ref{eq:What}). $\gamma$ is a learnable scaling parameter controlling the contribution of the original adapter. It is initialized per-layer as $\gamma_0 = \sigma\!\left(-\log \mathrm{RMS}(W)\right)$, where $\mathrm{RMS}(W)$ denotes the root-mean-square of the layer weights. This norm-adaptive scheme assigns smaller $\gamma$ to large-norm layers and larger $\gamma$ to small-norm layers, balancing their contributions, constrained to $(0,1)$ via the sigmoid function $\sigma$. The executing adapter $W_{\text{exec}}$ is used in the same manner as standard LoRA to produce the output $y$.

In this formulation, \(W_{\text{orig}}\) provides a stable representation of previously acquired knowledge, while \(\hat{W}\) introduces fast, input-specific adaptations through the selected programs. The coefficient $\gamma$ influences the stability–plasticity trade-off: smaller values allow greater flexibility for new information, while larger values generally emphasize retention of prior knowledge. This combination enables the model to adapt to new tasks while maintaining compatibility with previously learned representations, with \(W_{\text{orig}}\) serving as a stable reference that regularizes the optimization process.

During training, $W_{\text{orig}}$ is kept frozen and is not updated through backpropagation. Instead, gradient signals flow only through the programs $\hat{W}$, which is constructed from the adapter weight $W$. As a result, the adapter learns to produce adjustments relative to the fixed baseline provided by $W_{\text{orig}}$. This design stabilizes learning by preventing destructive updates to prior knowledge while allowing the programs to capture new information through localized modifications of the adapter weight.

While enabling rapid adaptation, the execution alone does not ensure stable retention of newly acquired knowledge and may introduce additional computation due to the construction of $W_{\text{exec}}$, affecting inference efficiency. To address this, we introduce a consolidation step:
\begin{equation}
W \leftarrow (1 - \lambda) W + \lambda W_{\text{exec}},
\end{equation}
where $\lambda \in [0,1]$ controls the consolidation rate. This update compresses input-conditioned program compositions into the persistent weight $W$. Specifically, $W_{\text{exec}} = \gamma W_{\text{orig}} + \hat{W}(x)$ captures fast, input-dependent adaptations, while $W$ accumulates their average effect over time.
This mirrors CLS theory: routing enables rapid, context-specific updates (fast memory), and consolidation distills them into a stable parametric representation (slow memory). As a result, $W$ evolves via both gradient updates and consolidation, acting as a compact summary of program-memory-based learning.

\begin{tcolorbox}[colback=gray!5,colframe=black]
\textbf{Remark 2: Consolidation Approximates the Expected Executing Adapter.}
Let $\Delta(x) := \hat{W}(x)$ denote the routed program composition, and define
\[
W^{*} := \gamma W_{\mathrm{orig}} + \mathbb{E}_{x \sim \mathcal{D}}[\Delta(x)].
\]
Under late-training stationarity of the routing map and bounded $\Delta(x)$,
the consolidation update drives $W$ toward $W^{*}$ in expectation within
each task.
\end{tcolorbox}

\begin{proof}[Proof sketch]
Recall that
\[
\Delta(x_t)=\sum_{n=1}^{N}\alpha_n(x_t)\sigma(s_n)W^{(n)}.
\]
Under late-training stationarity, $\mathbb{E}[\Delta(x)]$ is fixed. Using
\[
W^{(t+1)}=(1-\lambda)W^{(t)}
+\lambda\bigl(\gamma W_{\mathrm{orig}}+\Delta(x_t)\bigr),
\]
and defining $E^{(t)}=W^{(t)}-W^*$, where
$W^*=\gamma W_{\mathrm{orig}}+\mathbb{E}[\Delta(x)]$, we obtain
\[
E^{(t+1)}
=(1-\lambda)E^{(t)}
+\lambda\bigl(\Delta(x_t)-\mathbb{E}[\Delta(x)]\bigr).
\]
Taking expectation gives
\[
\mathbb{E}[E^{(t+1)}]=(1-\lambda)\mathbb{E}[E^{(t)}],
\]
so $\mathbb{E}[E^{(t)}]\to 0$ for $0<\lambda<1$, implying
$\mathbb{E}[W^{(t)}]\to W^*$. See Appendix~\ref{app:consolidation}.
\end{proof}

\textbf{Efficient Inference.} At inference time, we discard the routing mechanism and directly use the consolidated adapter $W$ (i.e., $W_{\text{exec}} = W$). The consolidation step integrates the executing adapters encountered during training into a single persistent parameter, allowing $W$ to approximate the average effect of input-conditioned program composition (Remark~2). This eliminates the need for program selection and storage for $W_{orig}$ at test time, while preserving the benefits of program-based learning.

\section{Experimental Results}

\subsection{Experimental Setup}

\begin{table*}[t]
\centering
\begin{tabular}{lccccc>{\columncolor{lightblue}}c}
\toprule
\textbf{LLM} & \textbf{Seq-LoRA} & \textbf{EWC} & \textbf{Replay} & \textbf{O-LoRA} & \textbf{DEAL} & \textbf{ProCL (Ours)} \\
\midrule
Llama-3.2-3B 
& 67.7 {\scriptsize$\pm$ 0.9} 
& 68.1 {\scriptsize$\pm$ 1.2}
& 67.6 {\scriptsize$\pm$ 1.1} 
& 68.1 {\scriptsize$\pm$ 0.3} 
& 66.4 {\scriptsize$\pm$ 1.6} 
& \textbf{69.8 {\scriptsize$\pm$  0.6}} \\

Llama-3.1-8B 
& 64.5 {\scriptsize$\pm$ 1.4} 
& 65.3 {\scriptsize$\pm$ 1.1} 
& 62.8 {\scriptsize$\pm$ 1.3} 
& \textbf{72.4 {\scriptsize$\pm$ 0.4}} 
& \textbf{72.1 {\scriptsize$\pm$ 2.1}} 
& \textbf{73.2 {\scriptsize$\pm$ 0.9}} \\
\midrule

Qwen3-4B 
& 52.5 {\scriptsize$\pm$ 1.9} 
& 54.5 {\scriptsize$\pm$ 0.6} 
& 54.8 {\scriptsize$\pm$ 1.9} 
& 56.6 {\scriptsize$\pm$ 0.2} 
& 68.2 {\scriptsize$\pm$ 0.7} 
& \textbf{72.5 {\scriptsize$\pm$ 0.2}} \\

Qwen3-8B 
& 59.1 {\scriptsize$\pm$ 1.7} 
& 58.7 {\scriptsize$\pm$ 0.8} 
& 58.8 {\scriptsize$\pm$ 1.5} 
& 58.6 {\scriptsize$\pm$ 7.2} 
& \textbf{72.9 {\scriptsize$\pm$ 1.5}} 
& \textbf{72.8 {\scriptsize$\pm$ 2.4}} \\
\midrule

Flan-T5-Base 
& 61.3 {\scriptsize$\pm$ 2.2} 
& 60.9 {\scriptsize$\pm$ 1.1} 
& 60.6 {\scriptsize$\pm$ 1.3} 
& 54.2 {\scriptsize$\pm$ 0.7} 
& 61.4 {\scriptsize$\pm$ 1.9} 
& \textbf{62.8 {\scriptsize$\pm$ 1.0}} \\

Flan-T5-Large 
& \textbf{62.0 {\scriptsize$\pm$ 1.7}} 
& 62.6  {\scriptsize$\pm$ 0.9} 
& \textbf{62.8 {\scriptsize$\pm$ 1.2}} 
& 61.6 {\scriptsize$\pm$ 1.0} 
& 55.9 {\scriptsize$\pm$ 8.7} 
& \textbf{63.4 {\scriptsize$\pm$ 1.7}} \\
\midrule

\textbf{Average} 
& 61.2 {\scriptsize$\pm$ 1.6}
& 61.7 {\scriptsize$\pm$ 1.0}
& 61.2 {\scriptsize$\pm$ 1.4}
& 61.9 {\scriptsize$\pm$ 1.6}
& 66.2 {\scriptsize$\pm$ 2.8}
& \textbf{69.1 {\scriptsize$\pm$ 1.1}} \\
\bottomrule
\end{tabular}
\caption{Final average accuracy across continual learning rounds on QA benchmarks. Bold denotes the best method, or those within a small effect size of it (Cohen's $d < 0.5$).}
\label{tab:qa_r}
\end{table*}

\textbf{Benchmarks.} The benchmark consists of two continual learning settings. First, we construct a \textit{question answering} (QA) setting with three QA datasets (BoolQ$\rightarrow$SQuAD$\rightarrow$AdversarialQA) presented sequentially, each sampled to 5,000 training and 1,000 test examples. They are ordered to maximise the distribution shift, increasing difficulty. Since each task requires the model to generate an answer string matched against a ground-truth response, we report the mean per-CL-round accuracy, averaged first across tasks within each continual learning round and then across all rounds. Second, following \citet{han2025data}, we adopt a \textit{text classification} (TC) setting with three curricula of increasing length: 3-task, 4-task, and 15-task. For this setting, we report Average Accuracy (AA), the mean classification accuracy across tasks per continual learning round averaged over all rounds, and ROUGE-1 (R-1), the unigram F1 overlap between model outputs and ground-truth labels, computed in the same manner.

\textbf{Baselines.} We evaluate ProCL against modern CL baselines: Sequential LoRA (Seq-LoRA), which initialises 
adapters from the previously learned matrices; O-LoRA~\citep{wang2023orthogonal}, which enforces 
orthogonality between task adapters; and DEAL~\citep{han2025data}, a state-of-the-art continual 
learning method for LoRA. For reference, in QA tasks, we also include classical continual learning methods adapted to LoRA parameters, including EWC \cite{kirkpatrick2017overcoming} and Replay \cite{rolnick2019experience}. For TC tasks, prior work has shown that modern CL methods consistently outperform classical approaches \cite{han2025data}; therefore, we omit the classical baselines to reduce computational cost.

\paragraph{Implementation details.} We experiment with three LLM families spanning a range of scales: 
Llama3 (3B, 8B), Qwen3 (4B, 8B), and Flan-T5 (Base, Large) (see Appendix Table \ref{tab:llm}). All baselines were re-run under our unified training and evaluation pipeline based on DEAL's official codebase, with the same LoRA rank and common hyperparameters across methods to ensure fair comparison. The experiments were conducted on either a single NVIDIA V100 GPU with 40GB memory or a single NVIDIA H100 GPU with 80GB memory.  Results are averaged over 3 runs. Statistical ties are declared using Cohen's $d < 0.5$ (small effect size). Full implementation details, including training configurations, hyperparameters, and dataset names, are provided in Appendix \ref{app:core-hparams}. For ProCL, if not stated otherwise, the default hyperparameters are $N=4$, $k_d=16$, and $\lambda=0.9$.

\subsection{Benchmarking Results}

\begin{table}[t]
\centering
\small
\setlength{\tabcolsep}{3.2pt}
\begin{tabular}{l c cc cc cc c}
\toprule
\multirow{2}{*}{Method} & \multirow{2}{*}{LLM} 
& \multicolumn{2}{c}{3-Task (TC)} 
& \multicolumn{2}{c}{4-Task (Standard)} 
& \multicolumn{2}{c}{15-Task (Large)} 
& \multirow{2}{*}{\begin{tabular}{c}Avg\\(AA)\end{tabular}} \\
& & AA & R-1 & AA & R-1 & AA & R-1 & \\
\midrule

Seq-LoRA & \multirow{4}{*}{\makecell[c]{LLaMA-3.1\\8B}}
& 60.8 {\scriptsize$\pm$ 0.2} & 62.1 {\scriptsize$\pm$ 0.5} 
& 55.3 {\scriptsize$\pm$ 0.3} & 57.0 {\scriptsize$\pm$ 0.6} 
& 49.6 {\scriptsize$\pm$ 0.2} & 52.4 {\scriptsize$\pm$ 0.4} 
& 55.2 {\scriptsize$\pm$ 0.3} \\
O-LoRA & 
& 59.7 {\scriptsize$\pm$ 0.5} & 61.0 {\scriptsize$\pm$ 0.3} 
& 53.4 {\scriptsize$\pm$ 0.4} & 55.2 {\scriptsize$\pm$ 0.7} 
& 48.9 {\scriptsize$\pm$ 0.2} & \textbf{52.8 {\scriptsize$\pm$ 0.5}} 
& 54.0 {\scriptsize$\pm$ 0.4} \\
DEAL & 
& 61.5 {\scriptsize$\pm$ 0.3} & \textbf{62.9 {\scriptsize$\pm$ 0.6}} 
& 56.4 {\scriptsize$\pm$ 0.2} & 58.2 {\scriptsize$\pm$ 0.4} 
& 50.8 {\scriptsize$\pm$ 0.5} & 52.6 {\scriptsize$\pm$ 0.3} 
& 56.2 {\scriptsize$\pm$ 0.3} \\
\rowcolor{lightblue}ProCL (Ours) & 
& \textbf{62.6 {\scriptsize$\pm$ 0.2}} & \textbf{62.5 {\scriptsize$\pm$ 0.7}} 
& \textbf{57.9 {\scriptsize$\pm$ 0.3}} & \textbf{59.6 {\scriptsize$\pm$ 0.5}} 
& \textbf{52.1 {\scriptsize$\pm$ 0.1}} & \textbf{53.7 {\scriptsize$\pm$ 0.4}} 
& \textbf{57.5 {\scriptsize$\pm$ 0.2}} \\
\midrule

Seq-LoRA & \multirow{4}{*}{\makecell[c]{LLaMA-3.2\\3B}}
& 56.1 {\scriptsize$\pm$ 0.6} & 57.8 {\scriptsize$\pm$ 0.4} 
& 50.4 {\scriptsize$\pm$ 0.2} & 52.9 {\scriptsize$\pm$ 0.5} 
& 45.2 {\scriptsize$\pm$ 0.3} & 48.0 {\scriptsize$\pm$ 0.6} 
& 50.6 {\scriptsize$\pm$ 0.3} \\
O-LoRA & 
& 54.8 {\scriptsize$\pm$ 0.5} & 56.5 {\scriptsize$\pm$ 0.3} 
& 48.3 {\scriptsize$\pm$ 0.4} & 50.6 {\scriptsize$\pm$ 0.2} 
& 44.1 {\scriptsize$\pm$ 0.7} & 48.5 {\scriptsize$\pm$ 0.4} 
& 49.1 {\scriptsize$\pm$ 0.5} \\
DEAL & 
& 56.9 {\scriptsize$\pm$ 0.2} & \textbf{58.6 {\scriptsize$\pm$ 0.5}} 
& 51.3 {\scriptsize$\pm$ 0.6} & 53.8 {\scriptsize$\pm$ 0.3} 
& 46.0 {\scriptsize$\pm$ 0.1} & 48.2 {\scriptsize$\pm$ 0.6} 
& 51.4 {\scriptsize$\pm$ 0.2} \\
\rowcolor{lightblue}ProCL (Ours) & 
& \textbf{58.0 {\scriptsize$\pm$ 0.3}} & 58.2 {\scriptsize$\pm$ 0.6} 
& \textbf{52.9 {\scriptsize$\pm$ 0.2}} & \textbf{55.1 {\scriptsize$\pm$ 0.5}} 
& \textbf{47.3 {\scriptsize$\pm$ 0.3}} & \textbf{49.4 {\scriptsize$\pm$ 0.4}} 
& \textbf{52.7 {\scriptsize$\pm$ 0.2}} \\
\midrule

Seq-LoRA & \multirow{4}{*}{\makecell[c]{Flan-T5\\Large}}
& 86.4 {\scriptsize$\pm$ 0.2} & 88.1 {\scriptsize$\pm$ 0.3} 
& 77.2 {\scriptsize$\pm$ 0.7} & 81.2 {\scriptsize$\pm$ 0.4} 
& 72.6 {\scriptsize$\pm$ 0.2} & 77.8 {\scriptsize$\pm$ 0.6} 
& 78.7 {\scriptsize$\pm$ 0.3} \\
O-LoRA & 
& 85.0 {\scriptsize$\pm$ 0.6} & 86.9 {\scriptsize$\pm$ 0.5} 
& 71.0 {\scriptsize$\pm$ 0.7} & 73.1 {\scriptsize$\pm$ 0.6} 
& 70.6 {\scriptsize$\pm$ 0.4} & \textbf{80.1 {\scriptsize$\pm$ 0.5}} 
& 75.5 {\scriptsize$\pm$ 0.6} \\
DEAL & 
& 87.5 {\scriptsize$\pm$ 0.2} & \textbf{89.2 {\scriptsize$\pm$ 0.4}} 
& 78.3 {\scriptsize$\pm$ 0.5} & 82.3 {\scriptsize$\pm$ 0.3} 
& 73.7 {\scriptsize$\pm$ 0.1} & 78.9 {\scriptsize$\pm$ 0.5} 
& 79.8 {\scriptsize$\pm$ 0.2} \\
\rowcolor{lightblue}ProCL (Ours) & 
& \textbf{88.3 {\scriptsize$\pm$ 0.3}} & \textbf{88.9 {\scriptsize$\pm$ 0.6}} 
& \textbf{81.0 {\scriptsize$\pm$ 0.4}} & \textbf{85.7 {\scriptsize$\pm$ 0.5}} 
& \textbf{74.6 {\scriptsize$\pm$ 0.2}} & \textbf{80.0 {\scriptsize$\pm$ 0.4}} 
& \textbf{81.3 {\scriptsize$\pm$ 0.3}} \\
\midrule

Seq-LoRA & \multirow{4}{*}{\makecell[c]{Flan-T5\\Base}}
& 75.7 {\scriptsize$\pm$ 0.3} & 77.5 {\scriptsize$\pm$ 0.5} 
& 66.8 {\scriptsize$\pm$ 0.6} & 76.3 {\scriptsize$\pm$ 0.4} 
& 55.6 {\scriptsize$\pm$ 0.2} & 63.0 {\scriptsize$\pm$ 0.3} 
& 66.0 {\scriptsize$\pm$ 0.4} \\
O-LoRA & 
& 75.8 {\scriptsize$\pm$ 0.5} & 77.6 {\scriptsize$\pm$ 0.2} 
& 66.4 {\scriptsize$\pm$ 0.3} & 76.0 {\scriptsize$\pm$ 0.6} 
& 53.4 {\scriptsize$\pm$ 0.4} & 62.0 {\scriptsize$\pm$ 0.2} 
& 65.2 {\scriptsize$\pm$ 0.3} \\
DEAL & 
& 77.0 {\scriptsize$\pm$ 0.2} & 78.7 {\scriptsize$\pm$ 0.6} 
& 67.2 {\scriptsize$\pm$ 0.3} & 76.8 {\scriptsize$\pm$ 0.2} 
& 56.2 {\scriptsize$\pm$ 0.5} & 63.4 {\scriptsize$\pm$ 0.3} 
& 66.8 {\scriptsize$\pm$ 0.3} \\ 
\rowcolor{lightblue}ProCL (Ours) & 
& \textbf{78.1 {\scriptsize$\pm$ 0.4}} & \textbf{79.8 {\scriptsize$\pm$ 0.5}} 
& \textbf{68.2 {\scriptsize$\pm$ 0.3}} & \textbf{77.8 {\scriptsize$\pm$ 0.2}} 
& \textbf{57.3 {\scriptsize$\pm$ 0.6}} & \textbf{64.4 {\scriptsize$\pm$ 0.4}} 
& \textbf{67.9 {\scriptsize$\pm$ 0.4}} \\
\bottomrule
\end{tabular}
\vspace{6pt}
\caption{Final average accuracy across continual learning rounds on text classification benchmarks. Results are reported as mean $\pm$ standard deviation over 3 runs. Avg (AA) denotes the mean accuracy across the three curricula. Bold denotes the best method, or all methods within a small effect size of it (Cohen's $d < 0.5$).}
\label{tab:continual_learning}
\end{table}

\textbf{QA Benchmarks.} Table~\ref{tab:qa_r} demonstrates that ProCL achieves the best overall performance with an average accuracy of 69.1, outperforming all baselines by a clear margin (e.g., +2.9 over DEAL and +7.9 over Seq-LoRA). While DEAL is competitive on some larger models, its performance is inconsistent across architectures. In contrast, ProCL delivers stable gains across all models, with particularly strong improvements in more challenging settings such as Qwen3-4B. These results highlight the robustness and effectiveness of ProCL in balancing stability and plasticity.

\textbf{Text Classification Benchmarks.}
Table~\ref{tab:continual_learning} shows that ProCL consistently achieves the best or second-best results across curricula, model scales, and both AA and R-1 metrics. Its advantage is more evident on longer task sequences: on the 15-task curriculum, ProCL improves AA over the strongest baseline, DEAL, by approximately 0.9--1.3\% across models. The gains are also clear in the 4-task setting, where ProCL outperforms DEAL by 2.7\% on Flan-T5-Large (81.0 vs.\ 78.3) and 1.0\% on Flan-T5-Base (68.2 vs.\ 67.2). Even in the shorter 3-task setting, ProCL maintains a consistent advantage, suggesting stronger early-stage retention. Across architectures, Flan-T5 generally achieves higher absolute accuracy than LLaMA, likely because its encoder--decoder design is better aligned with classification-style objectives; nevertheless, ProCL remains effective on both model families.

\subsection{Ablation Study}

In Table~\ref{tab:ablation}, we analyze the contribution of each component of ProCL across multiple backbone models on the QA benchmark. Removing the original adapter weight $W_{\mathrm{orig}}$ leads to substantial performance degradation on LLaMA-3.2-3B ($69.8 \rightarrow 65.9$) and Qwen3-4B ($72.5 \rightarrow 67.1$), highlighting the importance of maintaining a stable reference representation during continual adaptation. On Flan-T5-Base, disabling consolidation causes the largest drop ($62.8 \rightarrow 60.1$), indicating that transferring the dynamically composed execution weights into a persistent adapter is particularly important for encoder--decoder architectures. Replacing learned routing with random or uniform routing also consistently degrades performance across models, confirming that input-conditioned program selection is essential for reducing interference and improving retention.

\begin{table}[t]
\centering
\small
\begin{tabular}{lcccc>{\columncolor{lightblue}}c}
\toprule
Configuration 
& w/o $W_{\mathrm{orig}}$ 
& w/o Cons. 
& Random $\alpha$ 
& Uniform $\alpha$
& Full \\
\midrule
LLaMA-3.2-3B 
& 65.9 {\scriptsize$\pm$ 0.3} 
& 68.2 {\scriptsize$\pm$ 1.0} 
& 67.5 {\scriptsize$\pm$ 1.4} 
& 69.0 {\scriptsize$\pm$ 1.6} 
& \textbf{69.8 {\scriptsize$\pm$ 0.6}} \\
Qwen3-4B 
& 67.1 {\scriptsize$\pm$ 0.4} 
& 69.4 {\scriptsize$\pm$ 0.8} 
& 68.6 {\scriptsize$\pm$ 1.2} 
& 70.1 {\scriptsize$\pm$ 1.3} 
& \textbf{72.5 {\scriptsize$\pm$ 0.2}} \\
Flan-T5-Base 
& 61.6 {\scriptsize$\pm$ 0.8} 
& 60.1 {\scriptsize$\pm$ 1.3} 
& 61.8 {\scriptsize$\pm$ 1.2} 
& 60.4 {\scriptsize$\pm$ 0.9} 
& \textbf{62.8 {\scriptsize$\pm$ 1.0}} \\
\bottomrule

\end{tabular}
\vspace{6pt}
\caption{ProCL ablation on the QA benchmark across different backbone models.}
\label{tab:ablation}
\end{table}

\subsection{Model Analysis}\label{subsec:modelanal}

\textbf{Hyperparameter Sensitivity.}
Figure~\ref{fig:hptime} (left) reports the sensitivity of ProCL to the number of programs ($N$), key dimension ($d_k$), and consolidation strength ($\lambda$) on the QA benchmark using LLaMA-3.2-3B as the representative model. The default configuration, $N=4$, $d_k=16$, and $\lambda=0.9$, achieves the best performance of $69.78\pm0.60$. Using only a single program significantly reduces performance to $62.87\pm3.18$, while increasing the memory size to $N=16$ also degrades accuracy to $66.29\pm1.70$, suggesting that moderate program capacity is preferable. Similarly, $d_k=16$ outperforms both smaller and larger key dimensions, indicating a balance between expressiveness and stability. For consolidation, large but non-extreme values of $\lambda$ remain consistently effective, whereas full replacement updates ($\lambda=1$) noticeably reduce performance to $67.09\pm0.84$. Overall, ProCL is relatively robust to variations in $k_d$ and $\lambda$, with the best results achieved using balanced hyperparameter settings.

\begin{figure}[t]
    \centering\small
    \includegraphics[width=1\linewidth]{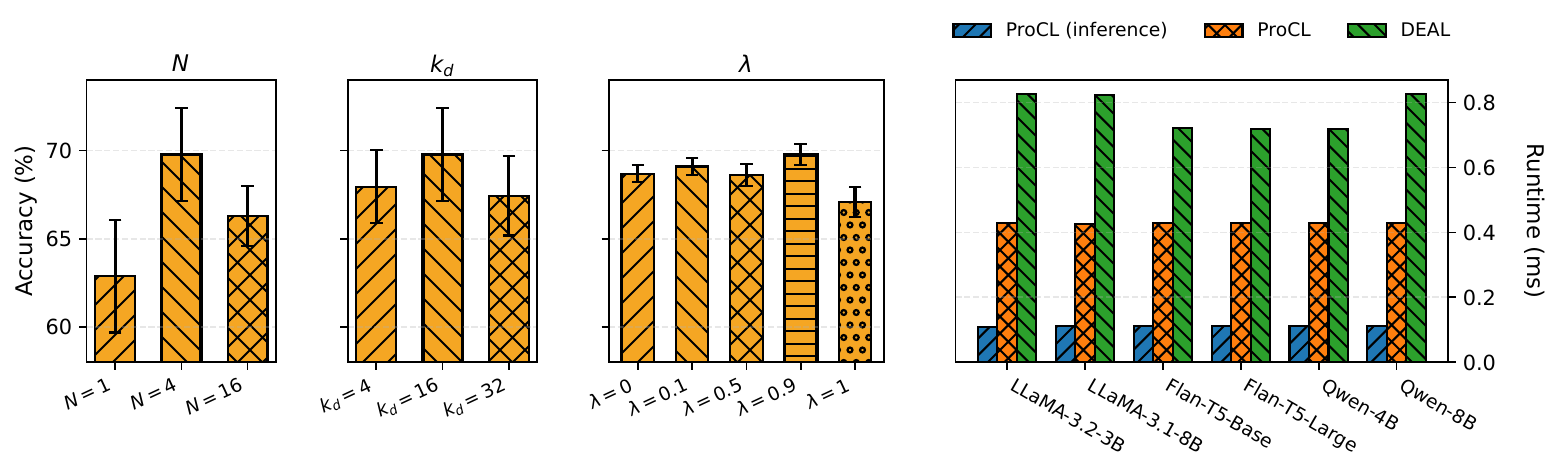}
    \caption{Left: Hyperparameter sensitivity of ProCL on QA benchmark. Right: Average runtime (ms per iteration) comparison of ProCL and DEAL  across LLMs}
    \label{fig:hptime}
\end{figure}

\noindent\textbf{On the Forgetting Issues.} We measure the forgetting on BoolQ  as the average accuracy drop after subsequent tasks: 
$\tfrac{1}{2}\big[(\mathrm{Acc}(\text{BoolQ} \mid \text{round 1}) 
- \mathrm{Acc}(\text{BoolQ} \mid \text{round 2})) 
+ (\mathrm{Acc}(\text{BoolQ} \mid \text{round 1}) 
- \mathrm{Acc}(\text{BoolQ} \mid \text{round 3}))\big]$, 
capturing both immediate and cumulative forgetting. As shown in Fig. \ref{fig:forgetting-bool}, existing CL methods exhibit noticeable forgetting, with accuracy drops typically ranging from moderate to severe across models (e.g., DEAL reaches 45.97 on Flan-T5-Large, while Seq-LoRA and O-LoRA often exceed 20 points on Qwen models). While DEAL achieves strong final accuracy on some architectures, its forgetting remains highly variable, reflecting a trade-off between plasticity and stability. In contrast, ProCL consistently reduces forgetting, achieving the lowest or near-lowest drop across all models (e.g., 4.63 on Flan-T5-Base and below 10 on Qwen3-4B/8B). Coupled with its superior final accuracy, this demonstrates that ProCL more effectively balances knowledge retention and adaptation compared to existing approaches.

\noindent\textbf{Efficiency.}
Fig. \ref{fig:hptime} (right) benchmarks ProCL against the second-best baseline DEAL across multiple LLM backbones on a single GPU, averaging over sequence lengths (128, 512) and batch sizes ($4, 16$) using 800 samples per model. 
With routing enabled during training, ProCL still runs $\sim$1.7--2.0$\times$ faster than DEAL and achieves a $\sim$6--8$\times$ speedup at inference by removing routing entirely.

\section{Conclusion}
We proposed ProCL, a memory-structured continual LoRA framework that organizes adaptation through program memory and consolidation. By enabling input-conditioned adapter composition during training and a compact static adapter at inference, ProCL achieves a favorable balance between plasticity and stability without additional deployment cost. Across diverse LLM backbones and tasks, our method consistently reduces catastrophic forgetting and improves retention over existing continual LoRA approaches. These results suggest that structuring the low-rank space with program-level memory is a promising direction for scalable continual adaptation of LLMs.

\section*{Limitations}

ProCL has two key limitations. First, its effectiveness relies on the emergence of routing specialization. When tasks are highly similar or share overlapping features, the routing may collapse to a few dominant programs, reducing the intended separation and limiting gains in mitigating interference. Second, ProCL assumes a fixed set of programs; as the number or diversity of tasks grows, this fixed capacity may become a bottleneck, forcing unrelated tasks to share programs and degrading long-term retention. Future work could improve routing robustness by encouraging program diversity (e.g., via entropy regularization) to prevent collapse under similar tasks. Additionally, extending ProCL with dynamic program expansion or pruning could alleviate fixed-capacity limitations as task diversity grows.

\section*{Broader Impacts}

ProCL improves continual fine-tuning by reducing catastrophic forgetting, which can enhance the efficiency and adaptability of deployed language models in evolving environments. This may benefit applications requiring incremental updates, such as personalized assistants or domain adaptation, by reducing retraining costs and improving knowledge retention. However, improved retention may also preserve outdated or biased information if not carefully managed. Additionally, more efficient continual learning could lower the barrier to deploying large models, potentially amplifying misuse. Future work should consider mechanisms for controlled forgetting, bias mitigation, and responsible deployment.

\bibliographystyle{plainnat}
\bibliography{references}

\newpage

\appendix
\renewcommand{\thesection}{\Alph{section}}  

\section*{Appendix}
\addcontentsline{toc}{section}{Appendix}

\section{Theoretical Results}
For notational simplicity, we describe the theoretical properties of our program memory framework in terms of a general weight matrix $W$. In practice, $W$ corresponds to the adapter weights, which are combined with the pretrained LLM weights to produce the final output. We omit this detail for brevity, as it does not affect the nature of the theoretical analysis. We also assume $\gamma = 1$ for notation simplicity (the result extends trivially by absorbing $\gamma$ into $W_{orig}$), without affecting the validity of the proof.

\subsection{Program Routing Reduces Cross-Task Interaction}\label{app:gradient_interference}

\subsection*{Setup and Assumptions}

Consider a continual learning sequence of tasks $\{1, \dots, T\}$. For each task $t$, let 
$x_t \in \mathbb{R}^{B \times D}$ be the batched input, $y_t^*$ the target, and $\mathcal{L}_t$ the 
task loss. Recall that $\hat{W} = [\hat{W}_1; \dots; \hat{W}_N]$ is formed by concatenating 
$N$ head outputs, where head $h$ produces:
\begin{equation}
    \hat{W}_h = \sum_{n=1}^{N} \bar{\alpha}_{h,n}\,\sigma(s_n)\,W^{(n)}, 
    \quad \hat{W}_h \in \mathbb{R}^{r \times D},\quad r = R/N
    \label{eq:head_output}
\end{equation}
and $\bar{\alpha}_{h,n} = \frac{1}{B}\sum_{b=1}^B \text{Softmax}(Q[h] K_h^\top)_{b,n}$ is 
the batch-averaged attention weight of head $h$ over program $n$. The executing weight is 
$W_{\text{exec}} = W_{\text{orig}} + \hat{W}$, with $W_{\text{orig}}$ frozen throughout training.

For simplicity, the output is  $y = xW_{\text{exec}}^\top$, so the gradient of the loss with respect to 
$W_{\text{exec}}$ is:
\begin{equation}
    \frac{\partial \mathcal{L}_t}{\partial W_{\mathrm{exec}}}
= (x_t W_{\mathrm{exec}}^\top - y_t^*)^\top x_t
\end{equation}

Since $W_{\text{orig}}$ is frozen, we define the residual target:
\begin{equation}
    r_t \triangleq y_t^* - x_t W_{\text{orig}}^\top
    \label{eq:residual}
\end{equation}
which isolates the novel task signal not yet explained by prior knowledge in $W_{\text{orig}}$. 
The gradient then simplifies to:
\begin{equation}
   \frac{\partial \mathcal{L}_t}{\partial W_{\mathrm{exec}}}
= (x_t \hat{W}^\top - r_t)^\top x_t
    \label{eq:exec_grad}
\end{equation}

Since $W^{(n)} \in \mathbb{R}^{r \times D}$ contributes to every head $h$ via 
$\bar{\alpha}_{h,n}$, and $\hat{W}_h$ occupies the $h$-th row block of $\hat{W}$, 
the gradient with respect to program $W^{(n)}$ accumulates across all heads:
\begin{equation}
    \frac{\partial \mathcal{L}_t}{\partial W^{(n)}} 
    = \sum_{h=1}^{N} \frac{\partial \mathcal{L}_t}{\partial \hat{W}_h}
      \cdot \frac{\partial \hat{W}_h}{\partial W^{(n)}}
    = \sigma(s_n) \sum_{h=1}^{N} \bar{\alpha}_{h,n} \cdot
      \big[(x_t \hat{W}^\top - r_t)^\top x_t\big]_h
    \label{eq:prog_grad}
\end{equation}
where $[\cdot]_h$ denotes the $h$-th row block of size $r \times D$. Program $n$ thus 
receives a \emph{head-weighted aggregation} of gradient signals, gated jointly by 
$\sigma(s_n)$ and the per-head routing weights $\{\bar{\alpha}_{h,n}\}_{h=1}^N$.

We analyze \textbf{ProCL} that optimizes programs $\{W^{(n)}\}$ via Eq.~\eqref{eq:prog_grad}, 
          where $W_{\text{orig}}$ anchors gradients to the residual $r_t$, and 
          $\{\bar{\alpha}_{h,n}\}$ selectively gates updates per program per head.
We make the following assumption:
\begin{itemize}
    \item[\textbf{A1.}] \textbf{(Routing specialization)} As training proceeds, 
          attention sharpens so each task $t$ concentrates routing mass on a subset 
          $\mathcal{S}_t \subseteq \{1,\dots,N\}$, with 
          $\mathcal{S}_t \cap \mathcal{S}_{t'} \to \emptyset$ for tasks with 
          distinct input statistics.
\end{itemize}

\begin{proposition}[Program Routing Bounds Cross-Task Gradient Interaction]
\label{prop:interference}
Let $\mathcal{J}(t,t')$ denote the cross-task gradient interaction magnitude between tasks 
$t$ and $t' < t$:
\begin{equation}
    \mathcal{J}(t,t') \triangleq 
    \left|\left\langle 
        \frac{\partial \mathcal{L}_t}{\partial W_{\mathrm{exec}}},\,
        \frac{\partial \mathcal{L}_{t'}}{\partial W_{\mathrm{exec}}}
    \right\rangle_{\!\!F}\right| \geq 0,
\end{equation}
and let $\mathcal{S}(t,t') \triangleq \sum_h |\langle \tilde g_{h,t}, \tilde g_{h,t'} \rangle_F|$ 
aggregate per-head interaction magnitudes. Then \textbf{ProCL} satisfies
\begin{equation}
    \mathcal{J}_{\textup{ProCL}}(t,t') \;\leq\; \big(\max_h \beta_h\big)\, \mathcal{S}(t,t').
\end{equation}
Under Assumption~\textbf{A1}, $\max_h \beta_h \to 0$ as routing specialization 
increases, and consequently $\mathcal{J}_{\textup{ProCL}}(t,t') \to 0$.
\end{proposition}

\begin{proof}
Let $G^t = \frac{\partial \mathcal L_t}{\partial W_{\mathrm{exec}}}$ denote the effective adapter gradient of task $t$.
Since $\hat{W} = [\hat W_1, \dots, \hat W_N]$ is formed by concatenating
the outputs of $N$ heads along the row dimension, the gradient naturally
decomposes into disjoint row blocks. Let $\tilde g_{h,t}$ denote the gradient
block corresponding to head $h$, embedded in the full $R \times D$ space. 
Because different heads occupy non-overlapping row blocks, their supports are
disjoint, and thus for $h \neq h'$,
\[
\langle \tilde g_{h,t}, \tilde g_{h',t'} \rangle_F = 0.
\]
Therefore, the full inner product decomposes as
\[
\langle G^t , G^{t'} \rangle_F
=
\sum_h
\langle \tilde g_{h,t}, \tilde g_{h,t'} \rangle_F .
\]

Under ProCL, gradients are routed through programs as in Eq.~\ref{eq:prog_grad}, yielding
\[
\langle G^t , G^{t'} \rangle_F^{\text{ProCL}}
=
\sum_h
\beta_h
\langle \tilde g_{h,t}, \tilde g_{h,t'} \rangle_F ,
\]
where
\[
\beta_h =
\sum_n
\sigma(s_{n,t})\sigma(s_{n,t'})
\bar{\alpha}_{h,n,t}\bar{\alpha}_{h,n,t'} .
\]

Here, because $s_n$ is a task-specific parameter (Sec. 4.3), we write $s_{n,t}$ to denote the value learned during task $t$. We first bound $\beta_h$.
Define
\[
u_n = \sigma(s_{n,t})\bar{\alpha}_{h,n,t},
\qquad
v_n = \sigma(s_{n,t'})\bar{\alpha}_{h,n,t'} .
\]
Then $\beta_h = \sum_n u_n v_n$.
By the Cauchy--Schwarz inequality,
\[
\beta_h \le \|u\|_2 \|v\|_2 .
\]

Since $0<\sigma(\cdot)<1$ and routing weights form a probability distribution,
\[
\sum_n \bar{\alpha}_{h,n,t} = 1, \qquad 0 \le \bar{\alpha}_{h,n,t} \le 1 ,
\]
we obtain
\[
\|u\|_2^2
=
\sum_n \sigma(s_{n,t})^2 \bar{\alpha}_{h,n,t}^2
\le
\sum_n \bar{\alpha}_{h,n,t}^2
\le
\sum_n \bar{\alpha}_{h,n,t}
=
1 .
\]
Thus $\|u\|_2 \le 1$ and similarly $\|v\|_2 \le 1$, implying
\[
0 \le \beta_h \le 1 .
\]

Applying the triangle inequality to the routed inner-product decomposition,
\[
\mathcal{J}_{\text{ProCL}}(t,t')
=
\Big|\sum_h \beta_h \langle \tilde g_{h,t}, \tilde g_{h,t'} \rangle_F\Big|
\le
\sum_h \beta_h \big|\langle \tilde g_{h,t}, \tilde g_{h,t'} \rangle_F\big|.
\]
Bounding each $\beta_h$ by $\max_{h'} \beta_{h'}$,
\[
\mathcal{J}_{\text{ProCL}}(t,t')
\le
\big(\max_h \beta_h\big) \sum_h \big|\langle \tilde g_{h,t}, \tilde g_{h,t'} \rangle_F\big|
=
\big(\max_h \beta_h\big)\, \mathcal{S}(t,t').
\]

Finally, under Assumption~A1 (routing specialization), for each task $t$ and head $h$, the routing distribution $\bar{\alpha}_{h,\cdot,t}$ concentrates on a subset $\mathcal S_t \subseteq \{1,\dots,N\}$ such that
\[
\sum_{n \in \mathcal S_t} \bar{\alpha}_{h,n,t} \to 1,
\qquad
\sum_{n \notin \mathcal S_t} \bar{\alpha}_{h,n,t} \to 0,
\]
and for tasks with distinct input statistics, these subsets become asymptotically disjoint, i.e.,
\[
\mathcal S_t \cap \mathcal S_{t'} = \emptyset.
\]

Since $0 < \sigma(\cdot) \le 1$, we upper bound
\[
\beta_h
\le
\sum_{n}
\bar{\alpha}_{h,n,t}\bar{\alpha}_{h,n,t'}
=
\langle \bar{\alpha}_{h,\cdot,t}, \bar{\alpha}_{h,\cdot,t'} \rangle.
\]

We now show that this inner product vanishes. Decompose over $\mathcal S_t$:
\[
\sum_n \bar{\alpha}_{h,n,t}\bar{\alpha}_{h,n,t'}
=
\sum_{n \in \mathcal S_t}
\bar{\alpha}_{h,n,t}\bar{\alpha}_{h,n,t'}
+
\sum_{n \notin \mathcal S_t}
\bar{\alpha}_{h,n,t}\bar{\alpha}_{h,n,t'}.
\]

For the first term, since $\mathcal S_t \cap \mathcal S_{t'} = \emptyset$, all indices in $\mathcal S_t$ lie outside $\mathcal S_{t'}$, and thus
\[
\sum_{n \in \mathcal S_t}
\bar{\alpha}_{h,n,t}\bar{\alpha}_{h,n,t'}
\le
\Big(\sum_{n \in \mathcal S_t}
\bar{\alpha}_{h,n,t}\Big)
\cdot
\max_{n \in \mathcal S_t} \bar{\alpha}_{h,n,t'}
\;\to\; 0,
\]
since $\bar{\alpha}_{h,n,t'} \to 0$ for $n \notin \mathcal S_{t'}$.

For the second term,
\[
\sum_{n \notin \mathcal S_t}
\bar{\alpha}_{h,n,t}\bar{\alpha}_{h,n,t'}
\le
\sum_{n \notin \mathcal S_t}
\bar{\alpha}_{h,n,t}
\;\to\; 0.
\]

Combining both bounds yields
\[
\sum_n \bar{\alpha}_{h,n,t}\bar{\alpha}_{h,n,t'} \to 0,
\quad\text{and hence}\quad
\beta_h \to 0
\quad\text{for every } h.
\]

Therefore $\max_h \beta_h \to 0$. Since $\mathcal{S}(t,t')$ is a finite sum of finite-dimensional gradient inner products, the bound 
$\mathcal{J}_{\text{ProCL}}(t,t') \le (\max_h \beta_h)\, \mathcal{S}(t,t')$
gives
\[
\mathcal{J}_{\text{ProCL}}(t,t') \to 0.
\]
\end{proof}

\subsection{Theoretical Justification of Consolidation and Efficient Inference}
\label{app:consolidation}

We formalize the consolidation mechanism within a single task, under a
late-training regime where the program memory and routing behavior are
approximately stable. The adapter is structured as
\[
W = [W^{(1)}; \ldots; W^{(N)}],
\]
and the routed program composition is
\begin{equation}
\Delta(x) := \hat{W}(x) = \sum_{n=1}^{N} \alpha_n(x)\, \sigma(s_n)\, W^{(n)},
\end{equation}
where $\alpha_n(x)$ are routing weights with $\sum_n \alpha_n(x) = 1$. The
executing adapter is
\[
W_{\mathrm{exec}}(x) = \gamma W_{\mathrm{orig}} + \Delta(x),
\]
and consolidation updates
\[
W^{(t+1)} = (1-\lambda) W^{(t)} + \lambda W_{\mathrm{exec}}(x_t),
\qquad x_t \sim \mathcal{D},
\]
where $\lambda \in (0,1)$ is the consolidation rate.

\paragraph{Setup and Assumptions.}

Within the analysis window, we make the following assumptions:

\begin{itemize}
    \item[\textbf{A2.}] \textbf{(Late-training stationarity)} The map
    $x \mapsto \Delta(x)$ is approximately stationary: the slot drift
    induced by consolidation is small relative to the routing-induced
    averaging over the analysis window. We treat $\Delta(x)$ as a fixed
    function of $x$ for the purposes of the recursion analysis.

    \item[\textbf{A3.}] \textbf{(Sampling)} The samples $x_t \sim \mathcal{D}$
    are drawn i.i.d., so that $\mathbb{E}[\Delta(x_t) \mid W^{(t)}] =
    \mathbb{E}_{x \sim \mathcal{D}}[\Delta(x)]$.

    \item[\textbf{A4.}] \textbf{(Boundedness)} $\mathbb{E} \|\Delta(x)\|_F
    < \infty$, ensuring $W^{*}$ is well-defined and
    $\mathbb{E}[\Delta(x)]$ exists as a finite matrix.
\end{itemize}

\textit{Scope.} Algorithm~1 resets $W_{\mathrm{orig}} \leftarrow W$ at
task boundaries. Proposition~\ref{prop:consolidation} therefore applies
within a single task, where $W_{\mathrm{orig}}$ and $\gamma$ are held fixed.
Across tasks, the recursion restarts from a new anchor.

\begin{proposition}[Consolidation Approximates Expected Routed Adapter]
\label{prop:consolidation}
Define
\[
W^{*} := \gamma W_{\mathrm{orig}} + \mathbb{E}_{x \sim \mathcal{D}}[\Delta(x)].
\]
Under Assumptions~A2--A4, $\mathbb{E}[W^{(t)}] \to W^{*}$ as
$t \to \infty$.
\end{proposition}

\begin{proof}
Under A2, $\Delta(x)$ is a linear combination of fixed program slots:
\[
\Delta(x) = \sum_{n=1}^{N} \alpha_n(x)\, \sigma(s_n)\, W^{(n)}.
\]
By linearity of expectation, justified by A4,
\[
\mathbb{E}[\Delta(x)] = \sum_{n=1}^{N} \mathbb{E}[\alpha_n(x)]\,
\sigma(s_n)\, W^{(n)}.
\]

Define $E^{(t)} := W^{(t)} - W^{*}$. From the consolidation update,
\begin{align*}
E^{(t+1)}
&= (1-\lambda) W^{(t)} + \lambda\big(\gamma W_{\mathrm{orig}} + \Delta(x_t)\big)
   - W^{*} \\
&= (1-\lambda) W^{(t)} - (1-\lambda)\gamma W_{\mathrm{orig}}
   + \lambda \Delta(x_t) - \mathbb{E}[\Delta(x)] \\
&= (1-\lambda)\big(W^{(t)} - \gamma W_{\mathrm{orig}} - \mathbb{E}[\Delta(x)]\big)
   + \lambda\big(\Delta(x_t) - \mathbb{E}[\Delta(x)]\big) \\
&= (1-\lambda) E^{(t)} + \lambda \xi_t,
\end{align*}
where $\xi_t := \Delta(x_t) - \mathbb{E}[\Delta(x)]$.

Under A3, $\mathbb{E}[\xi_t \mid W^{(t)}] = 0$. Taking expectation,
\[
\mathbb{E}[E^{(t+1)}] = (1-\lambda)\, \mathbb{E}[E^{(t)}].
\]
Unrolling the recursion gives
\[
\mathbb{E}[E^{(t)}] = (1-\lambda)^t\, E^{(0)} \to 0
\quad \text{as } t \to \infty,
\]
since $\lambda \in (0, 1)$ implies $|1-\lambda| < 1$. Therefore
$\mathbb{E}[W^{(t)}] \to W^{*}$.
\end{proof}

\paragraph{Remark on the late-training regime.}
Proposition~\ref{prop:consolidation} is a within-task, late-training
result, applying once the program memory and routing are approximately
stable --- a condition that is natural as training converges. Under this
regime, consolidation performs online stochastic approximation of
$W^{*}$, with each program slot $W^{(n)}$ contributing proportionally
to its expected routing frequency $\mathbb{E}[\alpha_n(x)]$ and scaling
factor $\sigma(s_n)$: frequently activated and highly scaled programs
dominate the consolidated adapter. Outside this regime, the target
$W^{*}_t$ becomes time-varying; consolidation then tracks a slowly
drifting mean, with $\lambda$ governing the trade-off between tracking
speed and variance.

\section{Method Details}

\begin{algorithm}[t]
\caption{ProCL Training}
\label{alg:procl-train}
\begin{algorithmic}[1]
\Require Continual datasets $\{\mathcal{D}_1,\ldots,\mathcal{D}_T\}$, initial LoRA adapter $W\in\mathbb{R}^{R\times D}$, number of programs $N$, consolidation rate $\lambda$
\Ensure Final consolidated adapter $W$

\State Partition $W$ into programs $\{W^{(1)},\ldots,W^{(N)}\}$, where $W^{(n)}\in\mathbb{R}^{r\times D}$ and $r=R/N$
\State $W_{\mathrm{orig}} \leftarrow W$

\For{task $t=1$ \textbf{to} $T$}

    \For{mini-batch $(\mathbf{x},\mathbf{y}^{*})\sim \mathcal{D}_t$}

        \For{sample $b=1$ \textbf{to} $B$}
            \State $\mathbf{z}^{(b)} \leftarrow \frac{1}{T}\sum_{\ell=1}^{T}\mathbf{x}^{(b)}_{\ell}$
            \State $\mathbf{Q}^{(b)} \leftarrow f_{\mathrm{enc}}(\mathbf{z}^{(b)}) \in \mathbb{R}^{N\times d_k}$
        \EndFor

        \For{head $h=1$ \textbf{to} $N$}
            \For{sample $b=1$ \textbf{to} $B$}
                \State $\boldsymbol{\alpha}^{(b)}_h \leftarrow
                \mathrm{Softmax}\!\left(\mathbf{Q}^{(b)}[h]K_h^\top\right)$
            \EndFor

            \State $\bar{\boldsymbol{\alpha}}_h \leftarrow
            \frac{1}{B}\sum_{b=1}^{B}\boldsymbol{\alpha}^{(b)}_h$

            \State $\widehat{W}_h \leftarrow
            \sum_{n=1}^{N}\bar{\alpha}_{h,n}\sigma(s_n)W^{(n)}$
        \EndFor

        \State $\widehat{W} \leftarrow [\widehat{W}_1 \,\|\, \cdots \,\|\, \widehat{W}_N]$

        \State $W_{\mathrm{exec}} \leftarrow \gamma W_{\mathrm{orig}} + \widehat{W}$

        \State Compute predictions $y$ using $W_{\mathrm{exec}}$
        \State Compute loss $\mathcal{L}(\mathbf{y},\mathbf{y}^{*})$

        \State Update trainable parameters via backpropagation
        \Comment{$W_{\mathrm{orig}}$ is frozen}

        \State $W \leftarrow (1-\lambda)W + \lambda W_{\mathrm{exec}}$
        \Comment{Consolidation}

    \EndFor

    \State $W_{\mathrm{orig}} \leftarrow W$
    \Comment{Freeze consolidated adapter for next task}

\EndFor

\State \Return $W$
\end{algorithmic}
\end{algorithm}

\subsection{Training Configurations and Hyperparameters}\label{app:core-hparams}
Unless otherwise stated, we use LoRA-based continual fine-tuning, where the adapter from round $t\!-\!1$ initializes the adapter in round $t$.
For text classification (TC) continual experiments, we follow prior work exactly\footnote{\url{https://github.com/zzm-black/DEAL-Continuous-Low-Rank-Fine-Tuning}} (Creative Commons Attribution-NonCommercial 4.0 International license), using a constant learning rate with no warmup and sequence-length limits of 512/50 (source/target). For QA tasks, we use the same base learning rate and epoch count with task-specific generation lengths. The default hyperparameters for all baselines are reported in Table \ref{tab:corehp}. A single continual learning run typically takes 2--24 hours on a V100 and 1--10 hours on an H100, depending on model size and task sequence.

\textbf{Datasets.}
For TC continual learning, we use up to 15 datasets:
\texttt{yelp}, \texttt{amazon}, \texttt{MNLI}, \texttt{CB}, \texttt{COPA}, \texttt{QQP}, \texttt{RTE}, \texttt{IMDB}, \texttt{SST-2}, \texttt{dbpedia}, \texttt{agnews}, \texttt{yahoo}, \texttt{MultiRC}, \texttt{BoolQA}, and \texttt{WiC}. 
For QA continual learning, we use up to 3 datasets: \texttt{google/boolq}, \texttt{squad}, and \texttt{adversarial\_qa}.

\begin{table}[t]
\centering

\begin{tabular}{lcc}
\toprule
\textbf{Model} & \textbf{Size} & \textbf{Hugging Face Link} \\
\midrule
LLaMA-3.2 & 3B & \url{https://huggingface.co/meta-llama/Llama-3.2-3B-Instruct} \\
LLaMA-3.1 & 8B & \url{https://huggingface.co/meta-llama/Llama-3.1-8B-Instruct} \\
Qwen3 & 4B & \url{https://huggingface.co/Qwen/Qwen3-4B} \\
Qwen3 & 8B & \url{https://huggingface.co/Qwen/Qwen3-8B} \\
Flan-T5 & Base & \url{https://huggingface.co/google/flan-t5-base} \\
Flan-T5 & Large & \url{https://huggingface.co/google/flan-t5-large} \\
\bottomrule
\end{tabular}
\caption{LLM used in our experiments. All models are publicly available on Hugging Face.}\label{tab:llm}
\end{table}

\begin{table}[t]
\centering
\small
\begin{tabular}{ll}
\toprule
\textbf{Hyperparameter} & \textbf{Value} \\
\midrule
Backbone models & Llama-3 / Flan-T5 / Qwen3 \\
PEFT method & LoRA (continual adapter initialization) \\
LoRA rank $r$ & 32 (UIE), 16 (QA) \\
LoRA scaling $\alpha$ & 32 \\
LoRA dropout & 0.1 \\
Learning rate & $1\times10^{-5}$ \\
Train batch size (per device) & 4 (TC), 16 (QA) \\
Gradient accumulation steps & 2 (TC), 1 (QA) \\
Number of epochs & 1 per task (except for TC's MNLI: 2) \\
Input / output length & 512 / 50 (TC), QA decoding up to 32 new tokens \\
\bottomrule
\end{tabular}
\vspace{6pt}
\caption{Core training hyperparameters used across experiments.}\label{tab:corehp}
\end{table}


\end{document}